\documentclass[wcp]{jmlr}

\usepackage{longtable}% for long tables
\usepackage{algorithm}
\usepackage{algorithmic}
\usepackage{booktabs}
\jmlrvolume{xx}
\jmlryear{2020}
\jmlrworkshop{ACML 2020}

\title[A New Accelerated Stochastic Gradient Method with Momentum]{A New Accelerated Stochastic Gradient Method with Momentum}

  \author{\Name{liang liu} \Email{mf1915058@smail.nju.edu.cn}\AND
  \Name{Xiaopeng Luo} \Email{xpluo@nju.edu.cn}\\
  \addr Nanjing China
 }

\editors{Wee Sun Lee and Taiji Suzuki}

\begin{document}

\maketitle

\begin{abstract}
In this paper, we propose a novel accelerated stochastic gradient method with momentum, which momentum is the weighted average of previous gradients. The weights decays inverse proportionally with the iteration times. Stochastic gradient descent with momentum (Sgdm) use weights that decays exponentially with the iteration times to generate an momentum term.
Using exponentially decaying weights, variants of Sgdm with well designed and complicated formats have been proposed to achieve better performance.
The momentum update rules of our method is as simple as that of Sgdm.
We provide theoretical convergence properties analyses for our method, which show both the exponentially decay weights and our inverse proportionally decay weights can limit the variance of the moving direction of parameters to be optimized to a region.
Experimental results empirically show that our method works well with practical problems and outperforms Sgdm, and it outperforms Adam in convolutional neural networks.
\end{abstract}

\section{Introduction}

A stochastic optimization problem can be consider as minimizing a differentiable function $F: \mathbb{R}^d \rightarrow \mathbb{R}$ using stochastic gradient descent (SGD) optimizer over a data set $\mathcal{S}$ which has $n$ samples. Function $F(x_k)$, in the empirical risk format, can be written as $F(x_k) = \frac{1}{n}\sum_{i=1}^{n}f_i(x_k)$, where $f_i(x_k)$ is the realization of $F(x_K)$ with the $i$th sample. Given the initial $x_0$ and a stepsize $\alpha_{k=0} >0$, the optimizer iterates with the following rule 
\begin{equation}
x_{k+1} \leftarrow x_k - \alpha_kg_i(x_k),
\end{equation}
until $F(x_k)$ reaches a predefined state.
The vector $g_i(x_k)$ is the stochastic gradient of $f_i(x)$, which gradient satisfies $\mathbb{E}[g_i(x_k)] = \nabla F(x_k)$ and has bounded variance~\cite{bottou2018}. For large-scale data sets, the computational complexity of iterating with full gradient $F(x) = \frac{1}{n}\sum_{i=1}^{n}$ is unacceptable. It is more efficient to measure a single component gradient $\nabla f_{i}(x_k)$, where $i$ is the index of the uniformly selected sample from the $n$ samples, and move in the noisy direction $g_i(x_k) = \nabla f_{i}(x_k)$, than to move in the full gradient direction $\nabla F(x_k) = \frac{1}{n}\sum_{i=1}^{n}\nabla f_{i}(x_k)$ with a bigger stepsize~\cite{adagradstepsize,bottou2018}. However, choosing a feasible and suitable stepsize schedule $\{\alpha_k>0\}$ is difficult.
According to the work of Robbins and Monro~\cite{robbins1951}, the stepsize schedule should satisfy
\begin{equation}
\sum_{k=1}^{\infty}\alpha_k = \infty \quad \text{and} \quad \sum_{k=1}^{\infty}\alpha_k^2 < \infty
\end{equation}
to make $\lim_{k\rightarrow\infty}\mathbb{E}[\|\nabla F(x_k)\|^2] = 0$. This constraint makes full gradient descent methods converge slower than stochastic SGD methods.

Methods using gradients in previous iterations as momentum terms have been proposed to accelerate the convergence and show great benefits~\cite{adam,sgdm}. Yet, theoretical analyses of those method in stochastic settings are elusive, and the ways to generate the momentum term from previous gradients can be revised.

\subsection{Stochastic Gradient Methods With Momentum}
Stochastic gradient descent methods with momentum (Sgdm) update parameter $x$ with
\begin{equation}
x_{k+1} \leftarrow x_k - \alpha_kg_k(x_k) + \gamma m_k,
\label{sgdm_x}
\end{equation}
where $m_k$ is called a momentum term, and Sgdm updates the momentum term with
\begin{equation}
m_{k} = \gamma m_{k-1} - \alpha_{k-1} g_k(x_{k-1}).
\label{sgdm_m}
\end{equation}
If $\gamma$ is a constant, then we call it a exponentially decay rate. With~(\ref{sgdm_m}), we can rewrite the update rule~(\ref{sgdm_x}) as
\begin{equation}
x_{k+1} = x_k + \sum_{j=1}^{k}\alpha_j\gamma^{k-j}g_{j}(x_j),
\label{x_update}
\end{equation}
which shows that the moving direction here is a weighted average of all gradients, with exponential decay weights. We denote the weighted average of gradients as $v_k = \sum_{j=1}^{k}\alpha_j\gamma^{k-j}g_{j}(x_j)$. It is believed that the momentum term can reduce the variance~\cite{initialization2013}. Using a exponential decay rate is common and effective. According our following analyses, a exponential decay rate can limit the variance of $v_k$ to a region which determined by the constant $\gamma$.

{\bf Main Contribution}
First, we proposed a novel stochastic gradient descent method with inverse proportional decay rate momentum, which method dynamically adjusts the momentum to match the convergence of the stochastic optimization problem. Besides, our rigorous analyses prove that both the simple SGDM and our novel momentum term can limit the variance to a region $\gamma$ (Theorem~\ref{theorem2},Theorem~\ref{theorem3}). We list out two main theorems (informally) in the following.

For a differential function $F$ with $L$-Lipschitz gradient and the variance of gradient is limited,  Theorem~\ref{theorem2} implies that the momentum term with exponentially decay weights can limit the variance of $v_k$
\begin{displaymath}
\mathbb{V}[v_k]\le\frac{1}{1-\gamma^2}\alpha_0^2(M+2M_V(\|\nabla F(x_k)\|_2^2 + L^2D^2)).
\end{displaymath}
Theorem~\ref{theorem3} implies that the momentum term with our inverse proportional decay weights can limit the variance of $v_k$
\begin{displaymath}
\mathbb{V}[v_k]\le\frac{\alpha_0^2}{2\beta}\Big(M+2M_V\|\nabla F(x_k)\|_2^2+2M_VL^2D^2\Big).
\end{displaymath}

Extensive experiments in Section 4 shows that the robustness of our method extends from linear regression to practical model in deep learning problems.

\subsection{Previous Work}
Momentum method and its use within optimization problems has been studied extensively ~\cite{initialization2013, o1996, sgdm}. The classical momentum (CM) method ~\cite{polyak} accumulates a decaying sum of the previous updates of parmeter $x$ into a momentum term $m$ using equation~(\ref{sgdm_m}) and updates $x$ with~(\ref{sgdm_x}). CM uses constant hyperparameter $\gamma$ and learning rate $\alpha$. With this method, one can see that the steps tend to accumulate
contributions in directions of persistent descent, while directions that oscillate tend
to be cancelled, or at least remain small~\cite{bottou2018}. Thus, this method can make optimization algorithms move faster along dimensions of low curvature where the update is small and its direction is persistent and slower along turbulent dimensions where the update usually significantly changes its direction ~\cite{initialization2013}. 
The adaptive momentum method in~\cite{sgdm} uses decaying learning rate and adaptive $\beta = \max(0, 1- \alpha_0 \mu^2)$, where $\mu(t)$ is the network input at time $t$.  
\subsection{Our Method}

Our method updates $x_k$ with 
\begin{equation}
x_{k+1} \leftarrow x_k - \alpha_k g_{k}(x_k) + \gamma(k)m_k, \label{x_novel_update}
\end{equation} and it updates the momentum term $m_K$ with
\begin{equation}
m_k \leftarrow \gamma(k-1)m_{k-1} - \alpha_{k-1}g_{k-1}(x_{k-1}), \label{m_novel_update}
\end{equation} where
\begin{equation}
y(k) = \left(\frac{k}{k+1}\right)^\beta.
\label{gamma_novel_update}
\end{equation}
The hyperparameter $\beta$ is a predefined constant. Thus we can rewrite the updating rules (\ref{x_novel_update}),(\ref{m_novel_update}), and (\ref{gamma_novel_update}) as 
\begin{displaymath}
x_{k+1} \leftarrow x_k - \sum_{i=1}^{k}\alpha_ig_i(x_i)\left(\frac{i}{k}\right)^\beta.
\end{displaymath}
We denote the moving direction of $x_k$ as $v_k =- \sum_{i=1}^{k}\alpha_ig_i(x_i)\left(\frac{i}{k}\right)^\beta$. 

We use the same stepsize as shown in (adam) 
\begin{equation}
	\alpha_i = \frac{\alpha_0}{\sqrt{i}}
	\label{stepsize}
\end{equation}
And it is straightforward that $\sum_{i}^{\infty}\alpha_k = \infty$.

\begin{algorithm}[tb]
	\caption{Our method}
	\label{algorithm}
	\begin{algorithmic}
		\STATE{\bfseries Require:} Stepsize $\alpha$
		\STATE{\bfseries Require:} Factor of decay rates, $\beta$
		\STATE{\bfseries Require:} Stochastic objective function $f(x)$
		\STATE{\bfseries Require:} Initial parameter vector $x_0$
		\STATE Initialize momentum vector $m_0\leftarrow0$
		\STATE Initialize iteration times $k\leftarrow0$
		\REPEAT
		\STATE $k \leftarrow k+1$
		\STATE $\alpha_{k} \leftarrow \frac{\alpha_0}{\sqrt{k}}$
		\STATE Generate gradient vector $g_k\leftarrow \nabla_x f_k(x_{k-1})$
		\STATE $m_k \leftarrow \gamma(k) m_{k-1} - \alpha_{k}g_k$
		\STATE $x_k \leftarrow x_{k-1} + m_k$
		\UNTIL{predefined conditions are achieved}
		\RETURN $x_k$
	\end{algorithmic}
\end{algorithm}

\section{Algorithm}
See Algorithm~\ref{algorithm} for the pseudo-code of our proposed algorithm. We denote $f(x)$ as a risk function with noise, which is differentiable with respect to parameters $x$. The optimization problem is to minimize the expected risk function $\mathbb{E}[f(x)]$ with respect to parameters $x$. We denote $f_1(x),\dots,f_T(x)$ as the realizations of the risk function at subsequent iterations $1,\dots,T$. Let $g_t = \nabla_x f_t(x)$ be the gradient of $f_t(x)$ with respect to $x$ in the $t$th iteration. 

Our algorithm updates weighted averages of the previous gradients, denoted as $m_k$. The hyper-parameter $\beta\in(1,\infty)$ control the decay rates of the weighted averages. We initialize $m_k$ as a vector with all 0 elements.

\section{Convergence Analysis}
In this section, we show that the momentum term could reduce the variance of stochastic directions $v_k$ along the iterations progress. A fixed exponential decay factor $\gamma$ and a changing factor $\gamma(k)$ can both limit the variance within a small range.
% objective function! the only type!!! check

Let us begin with a basic assumption of smoothness of the objective function.

\newtheorem{assumption}{Assumption~}[section]

\begin{assumption}[Lipschitz-continuous gradients]\label{ac:ass:lcg}
	The objective function $F: \mathbb{R}^d\rightarrow \mathbb{R}$ is continuously differentiable and the gradient function of $F$, namely, $\nabla F:\mathbb{R}^d\rightarrow\mathbb{R}^d$, is Lipschitz continuous with Lipschitz constant $L>0$, i.e.,
	\begin{equation*}
	\|\nabla F(x)-\nabla F(\bar{x})\|_2\le L\|x-\bar{x}\|_2
	~~\textrm{for all}~~\{x, \bar{x}\}\subset\mathbb{R}^d.
	\end{equation*}
\end{assumption}

\iffalse
{\bf Assumption 3.1}~(Lipschitz-continuous objective gradients). {\it The objective function $F: \mathbb{R}^d\rightarrow \mathbb{R}$ is continuously differentiable and the gradient function of $F$, namely, $\nabla F:\mathbb{R}^d\rightarrow\mathbb{R}^d$, is Lipschitz continuous with Lipschitz constant $L>0$, i.e.,
	\begin{displaymath}
	\|\nabla F(x) - \nabla F(\bar{x}) \|_2\le L\|x - \bar{x}\|_2,\; for\; all\; \{x, \bar{x}\}\subset\mathbb{R}^d.
	\end{displaymath}   }
\fi

Assumption~3.1 is essential for convergence analyses of most gradient-based methods, which ensures that the gradient of $F$ does not change arbitrarily quickly with respect to the parameter vector $x$. Based on Assumption~3.1, we have
\begin{equation}
F(x)\le F(\bar{x}) + \nabla F (\bar{x})^T(x- \bar{x}) + \frac{1}{2}L\|x - \bar{x}\|_2^2,
\label{lip_inequal}
\end{equation}
which inequality holds for all $\{x, \bar{x}\}\subset\mathbb{R}^d$.

\begin{proof}
	Using Assumption 3.1, we have
	\begin{displaymath}
	\begin{aligned}
	F(x) &= F(\bar{x}) + \int_0^1\frac{\partial F(\bar{x}+t(x - \bar{x}))}{\partial t}dt\\
	&= F(\bar{x}) + \int_0^1\nabla F(\bar{x}+t(x-\bar{x}))^T(x-\bar{x})dt\\
	&=F(x) + \nabla F(\bar{x})^T(x-\bar{x}) +\int_0^1[\nabla F(\bar{x}) + t(x - \bar{x}) -\nabla F(\bar{x})]^T(x - \bar{x})dt \\
	&\le F(\bar{x}) + \nabla F(\bar{x})^T(x-\bar{x}) + \int_0^1 L \|t(x-\bar{x})\|_2\|x-\bar{x}\|_2dt,
	\end{aligned}
	\end{displaymath}
	from which the desired result follows.
	\end{proof}

\subsection{Variance Analysis}
We first consider a fixed decay factor $\gamma\in(0,1)$, which is used in many simple stochastic gradient descent methods with momentum.

\begin{assumption}\label{ac:ass:var}
	The objective function F and stochastic gradient $g_k(x_k)$ satisfy for all $k \in \mathbb{N}$, there exist scalars $M\ge0$ and $M_v \ge 0$ such that
	\begin{equation*}
	\mathbb{V}_{x_k}[g_i(x_k)] \le M + M_V\|\nabla F(x_k)\|_2^2.
	\end{equation*}
\end{assumption}

\begin{theorem}\label{theorem1}
	Under Assumptions \ref{ac:ass:lcg} and \ref{ac:ass:var}, suppose that the algorithm 1 satisfies $\|x_i-x_j\|\leqslant D$ for any $i,j\in\mathbb{N}$. Then 
	\begin{equation*}
	\|\nabla F(x_j)\|_2^2\leqslant2\|\nabla F(x_k)\|_2^2+2L^2D^2.
	\end{equation*}
	\end{theorem}

\begin{proof}
	According to Assumption \ref{ac:ass:lcg}, there is a diagonal matrix $\Lambda= \mathrm{diag}(\lambda_1,\cdots,\lambda_d)$ with $\lambda_i\in [-L,L]$ such that
	\begin{align*}
	\nabla F(x_j)=\nabla F(x_k+\delta_{j,k})=\nabla F(x_k)+\Lambda\delta_{j,k},
	\end{align*}
	where $x_j=x_k+\delta_{j,k}$; and by further noting that $\|\delta_{j,k}\|_2=\|x_j-x_k\|_2\leqslant D$, we have
	\begin{align*}
	\|\nabla F(x_j)\|_2^2\leqslant&\Big(\|\nabla F(x_k)\|_2
	+L\|\delta_{j,k}\|_2\Big)^2 \\
	\leqslant&\|\nabla F(x_k)\|_2^2+L^2\|\delta_{j,k}\|_2^2
	+2L\|\delta_{j,k}\|_2\|\nabla F(x_k)\|_2 \\
	\leqslant&2\|\nabla F(x_k)\|_2^2+2L^2\|\delta_{j,k}\|_2^2 \\
	\leqslant&2\|\nabla F(x_k)\|_2^2+2L^2D^2
	\end{align*}
	as claimed. 
\end{proof}

\begin{theorem}[a fixed decay factor]\label{theorem2}
	Under the conditions of Theorem~\ref{theorem1} and Assumption~\ref{ac:ass:var}, suppose that (i) the sequence of iterates $\{x_k\}$ is generated with (\ref{x_update}) using a fixed factor $\gamma \in (0,1)$ and a stepsize sequence $\{\alpha_k\}$, which sequence satisfies  $\alpha_k \ge \alpha_{k+1}$ for all $k \in \mathbb{N}$, and (ii) the sequence $\{x_k\}$ satisfies $\|x_i - x_j\|\le D$ for any $i,j\in\mathcal{N}$, then
	\begin{displaymath}
	\mathbb{V}[v_k]\le \frac{1}{1-\gamma^2}\alpha_0^2\big(M+2M_V\|\nabla F(x_k)\|_2^2+2M_VL^2D^2\big)
	\end{displaymath}
\end{theorem}
\begin{proof}
	For the SGDM updating strategy, we have direction vector
	\begin{displaymath}
		 v_k = -\sum_{i=1}^{k}\alpha_i\gamma^{k-i} g_i(x_i).
	\end{displaymath}
	Hence, along with Assumption \ref{ac:ass:var}, we obtain
	\begin{align*}
	\mathbb{V}[v_k]=&\sum_{j=1}^k\gamma^{2(k-j)}\alpha_j^2
	\mathbb{V}_{x_j}[g_j(x_j)] \\
	\leqslant&\alpha_0^2\sum_{j=1}^k\gamma^{2(k-i)}
	\Big(M+M_V\|\nabla F(x_j)\|_2^2\Big) \\
	\leqslant&\sum_{i=1}^k\alpha_0^2\gamma^{2(k-i)}
	\Big(M+2M_V\|\nabla F(x_k)\|_2^2+2M_VL^2D^2\Big).
	\end{align*}
	Notice that 
	\begin{align*}
	\sum_{j=1}^k\gamma^{2(k-j)}=\frac{1-\gamma^{2k}}{1-\gamma^2}.
	\end{align*}
	Since $\frac{1-\gamma^{2k}}{1-\gamma^2}$ decays to $\frac{1}{1-\gamma^2}$ as $k$ increases for $0<\gamma<1$, so the variance of $v_k$ could be finally reduced to
	\begin{equation*}
	\frac{1}{1-\gamma^2}\alpha_0^2\Big(M+2M_V\|\nabla F(x_k)\|_2^2+2M_VL^2D^2\Big).
	\end{equation*}
\end{proof}

\begin{theorem}[a changing decay factor]\label{theorem3}
	Under the conditions of Theorem~\ref{theorem1} and Assumption~\ref{ac:ass:var}, suppose that (i) the sequence of iterates $\{x_k\}$ is generated with (\ref{x_novel_update}) using a changing factor $\gamma(k) \in (0,1)$ taking the form $\gamma_k = (k/k+1)^\beta$ and a stepsize sequence $\{\alpha_k\}$ as (\ref{stepsize}), which sequence satisfies  $\alpha_k \ge \alpha_{k+1}$ for all $k \in \mathbb{N}$, and (ii) the sequence $\{x_k\}$ satisfies $\|x_i - x_j\|\le D$ for any $i,j\in\mathcal{N}$, then
	\begin{displaymath} 
	\mathbb{V}[v_k]\le \frac{\alpha_0^2}{2\beta}\Big(M+2M_V\|\nabla F(x_k)\|_2^2+2M_VL^2D^2\Big)
	\end{displaymath}
\end{theorem}
\begin{proof}
	For our updating strategy
	\begin{displaymath}
	v_k = -\sum_{j=1}^{k}\frac{j^\beta}{k^\beta}\alpha_jg_j(x_j)
	\end{displaymath}
	Hence, along with Assumption \ref{ac:ass:var}, we obtain
	\begin{align*}
	\mathbb{V}[v_k]=&\alpha_0\sum_{j=1}^k\frac{j^{2\beta}}{k^{2\beta}}\frac{1}{j}
	\mathbb{V}_{\theta_j}[g_j(x_j)] \\
	\leqslant&\alpha_0^2\sum_{j=1}^k\frac{j^{2\beta}}{k^{2\beta}}\frac{1}{j}
	\Big(M+M_V\|\nabla F(x_j)\|_2^2\Big) \\
	\leqslant&\alpha_0^2\Big(M+2M_V\|\nabla F(x_k)\|_2^2+2M_VL^2D^2\Big)\frac{1}{k^{2\beta}}\int_{1}^{k}j^{2\beta-1}\\
	\leqslant& \frac{\alpha_0^2}{2\beta}\Big(M+2M_V\|\nabla F(x_k)\|_2^2+2M_VL^2D^2\Big),
	\end{align*}
	and the proof is complete.
\end{proof}

\section{Experiments}
To compare the performance of the proposed method with that of other methods, we investigate different popular machine learning models, including logistic regression, multi-layer fully connected neural networks and deep convolutional neural networks. Experimental results show that our novel momentum term can efficiently solve practical stochastic optimization problems in the field of deep learning and outperforms other methods in deep convolutional neural networks.

In our experiments, we use the same parameter initialization strategy for different optimization algorithms. The values of hyper-parameters are selected, based on results, from the common used settings, respectively.

\subsection{Experiment: Logistic Regression}
We evaluate the two methods (saying Sgdm, and our method), on two multi-class logistic regression problems using the MNIST dataset without regularization. Logistic regression is suitable for comparing different optimizers for its simple structure and convex objective function. In our experiments, we set the stepsize $\alpha_t = \frac{\alpha}{\sqrt{t}}$, where $t$ is the current iteration times. The logistic regression problems is to classify the class label directly on the $28\times28$ image matrix.

\begin{figure}[htbp]
	\centering
	\subfigure[]{\includegraphics[width=2.55in]{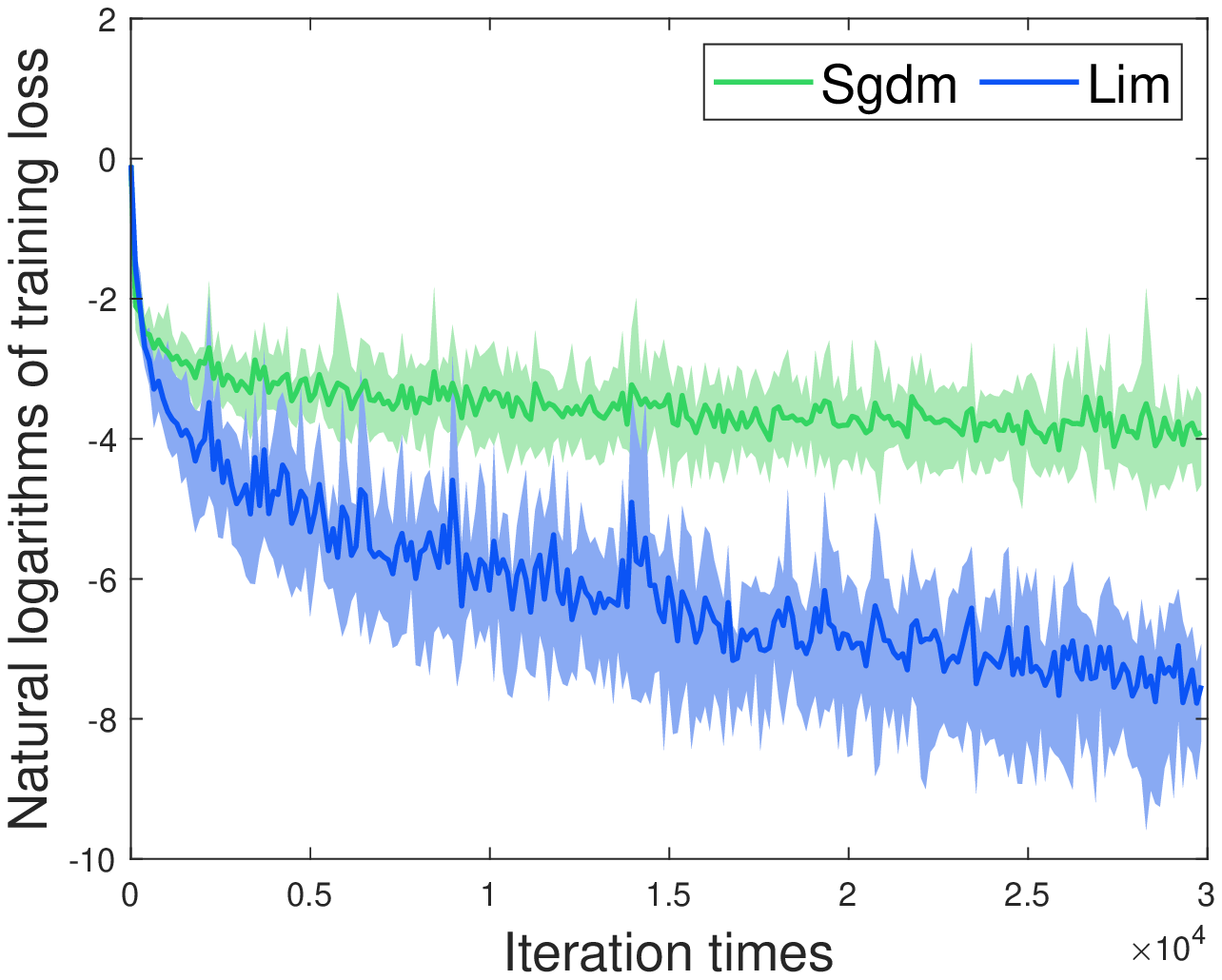}}
		\label{fig1a}
	\subfigure[]{\includegraphics[width=2.55in]{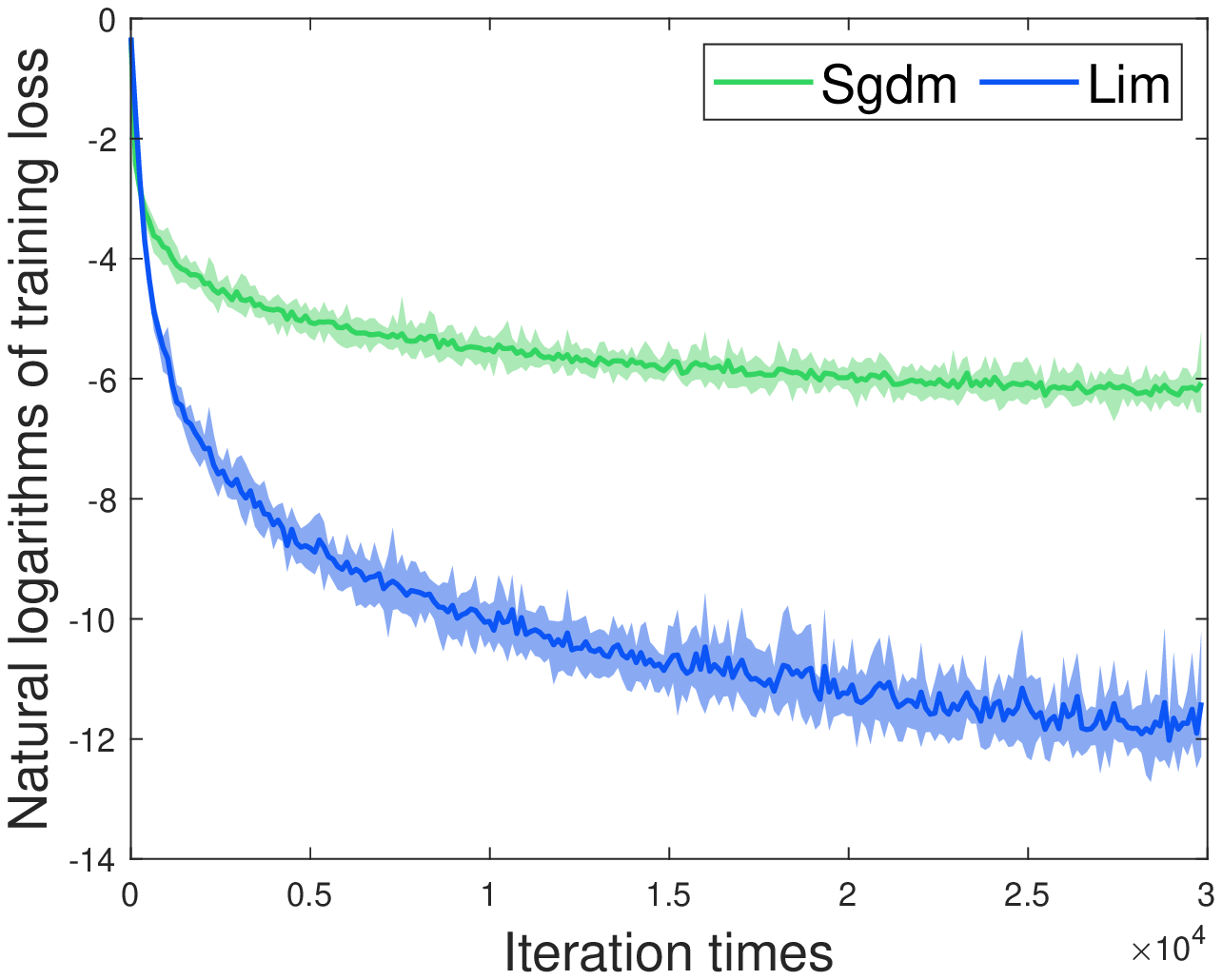}}
		\label{fig1b}
	\caption{The negative log likelihood training cost of logistic regression on MNIST images. (a) Normal logistic regression. (b) Logistic regression after a two layer fully connected neural network.}
\end{figure}

We compare the performance of our method (Lim for short) to optimize the logistic regression problem with that of Sgdm using a minibatch size of 128. According to Fig. 1, we found that our method converges faster than Sgdm.

\subsection{Experiment: Multi-layer Fully Connected Neural Networks}
Multi-layer neural networks are powerful models with non-convex objective functions. We empirically found that our method outperforms Sgdm. In our experiments, the multi-layer neural network models are consistent with that in previous publications, saying a neural network model with two or three fully connected hidden layers with 1000 hidden units each and ReLU activation are used for this experiment with a minibatch size of 128.
\begin{figure}[htp]
	\centering
	\subfigure[]{\includegraphics[width=2.55in]{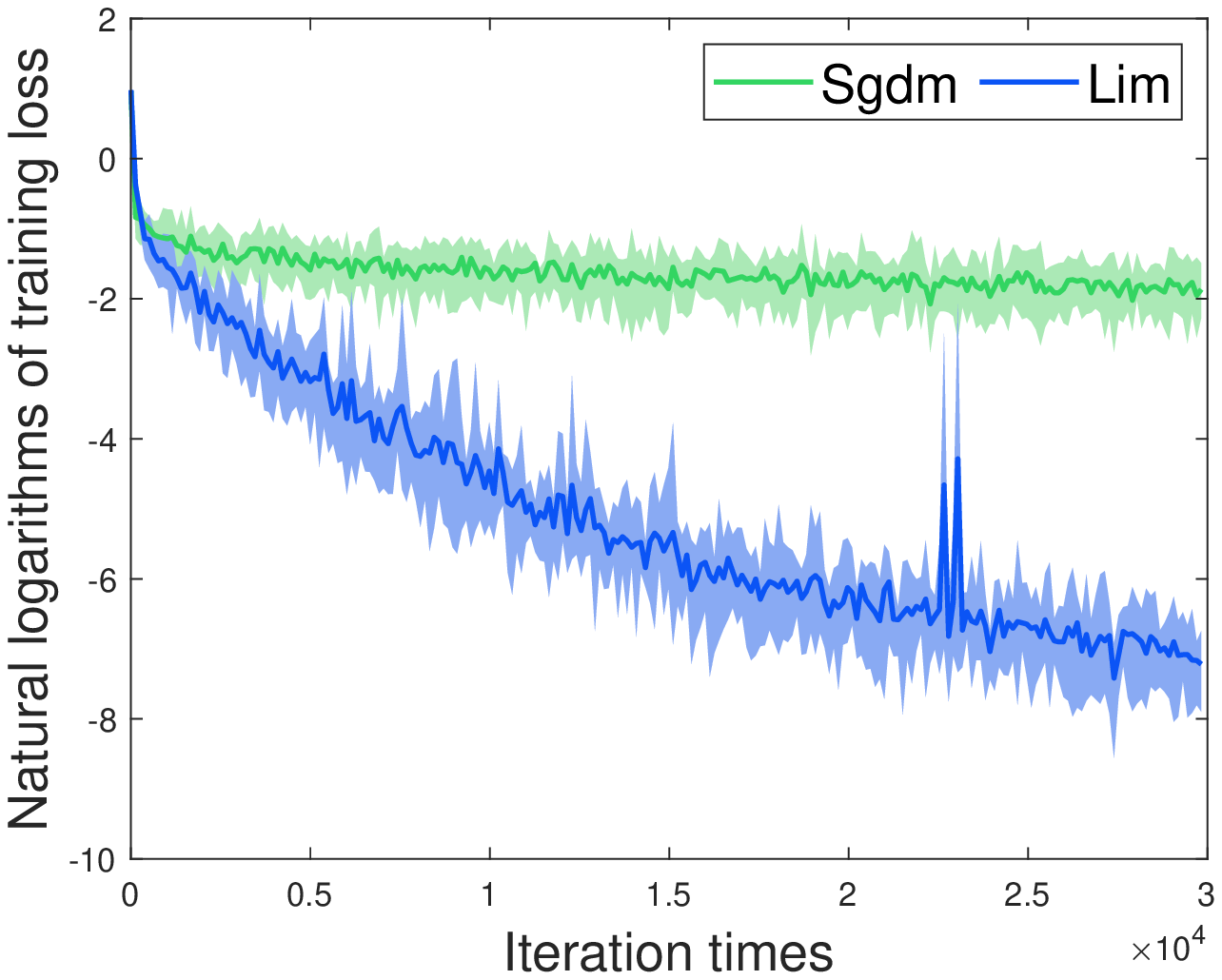}
		\label{fig9a}}
	\subfigure[]{\includegraphics[width=2.55in]{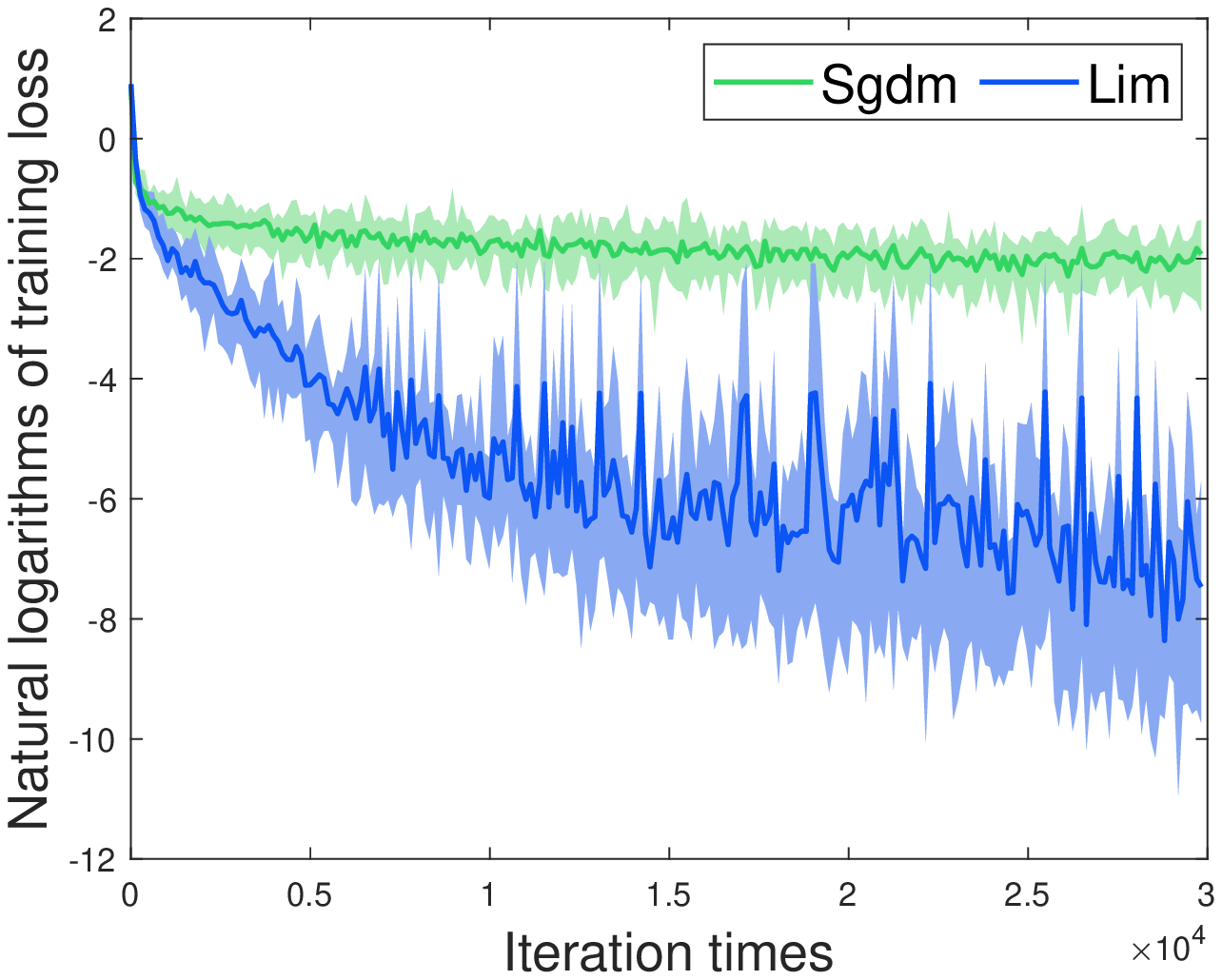}
		\label{fig9b}}
	\caption{Training cost of fully connected multilayer neural networks on MNIST images. (a) Neural network with two fully connected hidden layers.(b) Neural network with three fully connected hidden layers.}
	\label{fig9}
\end{figure}

We investigate the two optimizers using the standard deterministic cross-entropy objective function without regularization. According to Fig. 2, we found that our method  outperforms Sgdm by a large margin.

\subsection{Experiment: Deep Convolutional Neural Networks}
Deep convolutional neural networks (CNNs) have shown considerable success in practical machine learning tasks, e.g. computer vision tasks and natural language procession tasks. Our CNN architectures has three or two alternating stages of $5\times5$ convolution filters and $3\times3$ max pooling with stride of 2, each followed by a fully connected layer of 1000 units with RLeU activations. We pre-process the input image with whitening imply drop out noise to the input layer and fully connected layer. The minibatch size is also set to 128.

\begin{figure}[htp]
	\centering
	\subfigure[]{\includegraphics[width = 2.55in]{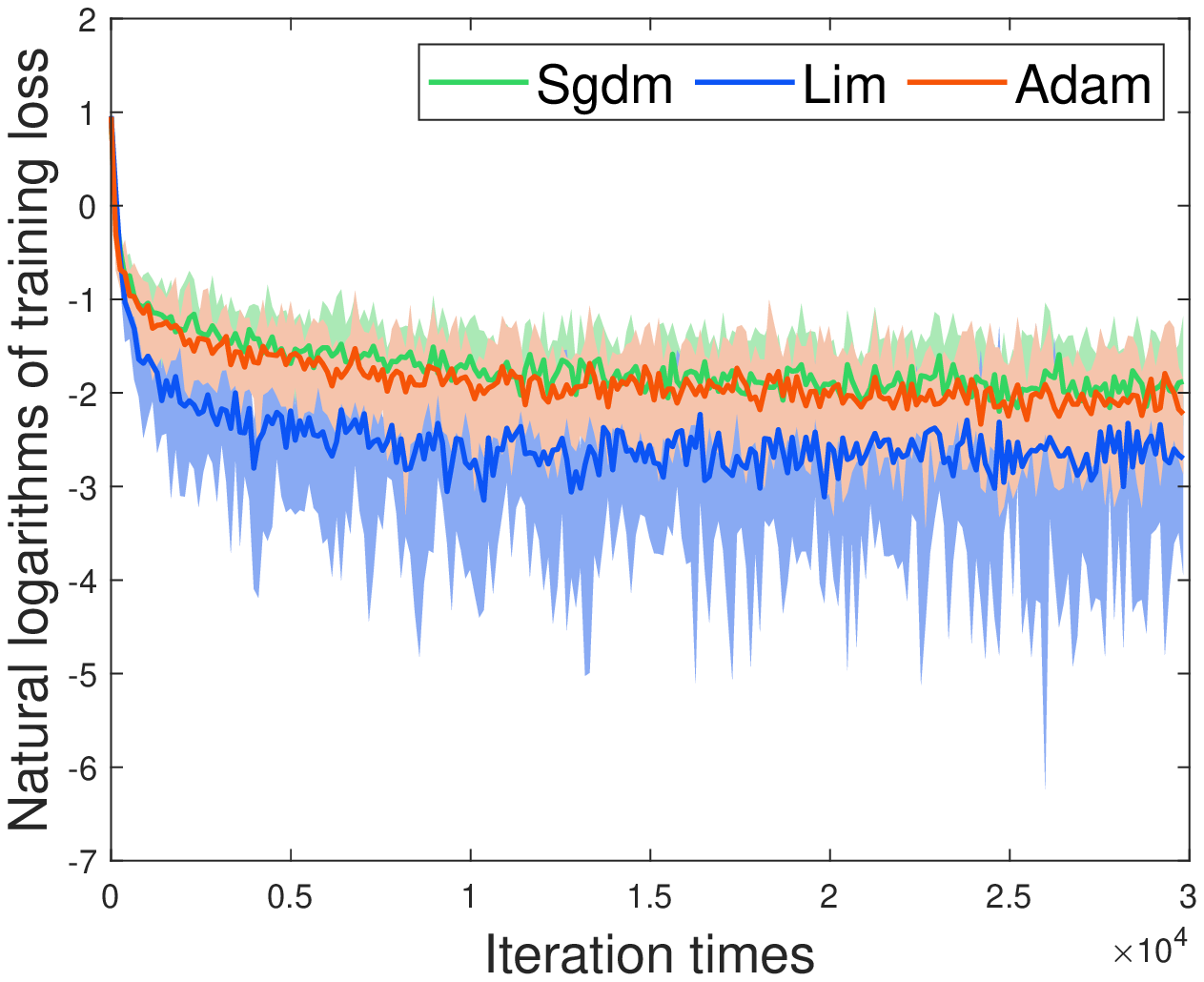}}
	\subfigure[]{\includegraphics[width = 2.55in]{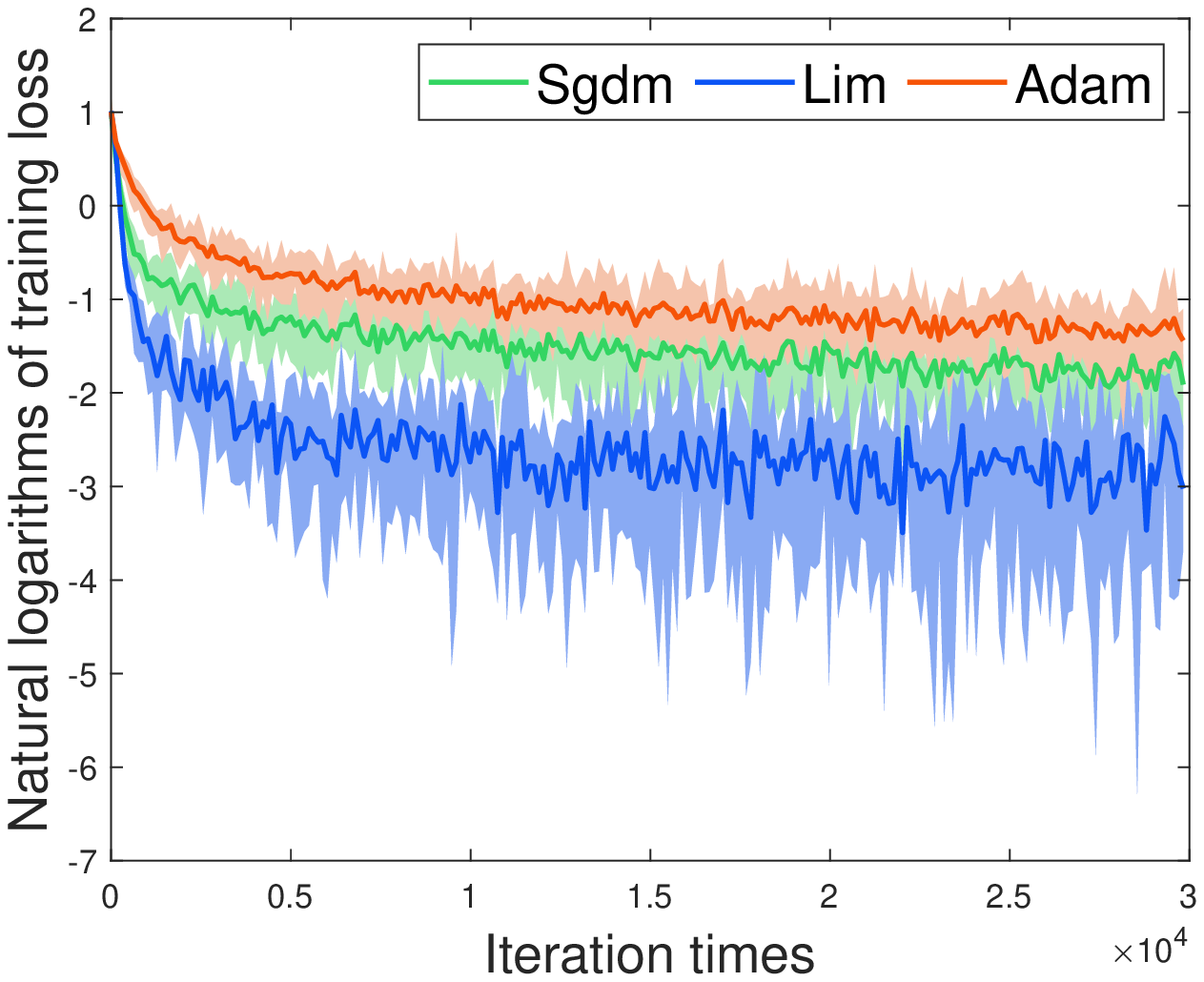}}
	\label{fig3}
	\caption{Training cost of deep convolutional neural networks on MNIST images. (a) CNN with two convolutional layers.(b) CNN with three convolutional layers.}
\end{figure}
 
We investigate the three optimizers using the standard deterministic cross-entropy objective function. According to Fig. 3, we found that our method converge faster than Adam and Sgdm.

\section{Conclusion}
In this paper, we introduced a computationally efficient algorithm for gradient-based stochastic optimization problems. The updating strategies of the proposed method enjoys the simplicity of the original SGD with momentum but utilizes the momentum more efficiently than SGD with momentum. Our method is designed for machine learning problems with large scale data sets and non-convex optimization problems, where is hard for stochastic optimizers to achieve linear converge speed. The experiment results confirm our theoretical analysis on its convergence property and show that our method can solve practical optimization problems efficiently. Overall, we found that our method is a robust and well-suited method to a wide range of non-convex optimization problems in the field of machine learning.

\bibliography{acml20}

\end{document}